\documentclass[twoside,11pt]{article}
\pdfoutput=1

\usepackage{jmlr2e}
\usepackage[utf8]{inputenc}
\usepackage{amsfonts}
\usepackage{amsmath}
\usepackage{float}
\usepackage[titlenumbered,ruled]{algorithm2e}
\usepackage{graphicx}
\usepackage{subfigure}
\graphicspath{ {figures/} }

\usepackage[backend=bibtex, maxbibnames=8]{biblatex}
\usepackage[bookmarks=false]{hyperref}

\usepackage{enumitem}
\setitemize{itemsep=5pt,parsep=0pt, topsep=0pt}
\setenumerate{itemsep=5pt,parsep=0pt, topsep=0pt}

\newcommand{\R}{\mathbb{R}}

\newcommand{\norm}[1]{\left\lVert#1\right\rVert}

\newcommand{\dotp}[2]{\langle #1, #2 \rangle}

\newcommand{\summ}[2]{\sum_{#1 = 1}^{#2}}

\newcommand{\eqdef}{\mathrel{\stackrel{\makebox[0pt]{\mbox{\normalfont\tiny def}}}{=}}}

\newcommand{\E}{\mathbb{E}}

\ShortHeadings{IKA: Independent Kernel Approximator}{Ronchetti}
\firstpageno{1}

\begin{document}

\title{IKA: Independent Kernel Approximator}

\author{\name Matteo Ronchetti \\ \email matteo@ronchetti.xyz \thanks{This work is part of my graduation thesis. Many thanks to professor Stefano Serra Capizzano.}}

\maketitle

\begin{abstract}%
This paper describes a new method for low rank kernel approximation called IKA. The main advantage of IKA is that it produces a function $\psi(x)$ defined as a linear combination of arbitrarily chosen functions. In contrast the approximation produced by Nystr\"om method is a linear combination of kernel evaluations. The proposed method consistently outperformed Nyström method in a comparison on the STL-10 dataset.
Numerical results are reproducible using source code available at \url{https://gitlab.com/matteo-ronchetti/IKA}
\end{abstract}

\section{Introduction}
Consider the problem of low rank kernel approximation which consists of approximating a kernel $K: \R^n \times \R^n \rightarrow \R$ with a function $\psi: \R^n \rightarrow \R^m$ such that $K(x,y) \approx \dotp{\psi(x)}{\psi(y)}$. This problem arises when:
\begin{itemize}
    \item Using a kernel method \cite{shawe,learn_with_kernels} on a Machine Learning problem where the dataset size renders operating on the full Gram matrix practically unfeasible;
    \item One wants to use the map $\psi(x)$ as a feature map for the construction of a multilayer model (such as Convolutional Kernel Networks \cite{ckn,sckn}).
\end{itemize}
In this paper we propose a new method for low rank kernel approximation called IKA, which has the following characteristics:
\begin{itemize}
    \item It produces a function $\psi(x) \in Span\{b_1(x),b_2(x),\dots,b_n(x)\}$ where the basis functions $b_i(x)$ can be arbitrarily chosen; the basis $b_i(x)$ is independent from the approximated kernel; 
    \item It is conceptually similar to Nystr\"om method \cite{nystrom} but in our experiments IKA produced better results (Section \ref{results}).
\end{itemize}

\newpage
\section{Preliminaries}
Let $\{x_1, x_2, \dots,x_N\}$ with $x_i \in \R^d$ be a dataset sampled i.i.d from an unknown distribution with density $p(x)$.
This density defines an inner product between real valued functions in $\R^d$
$$
\dotp{f}{g} \eqdef  \int_{\R^d} f(x)g(x)p(x)dx \qquad \norm{f} \eqdef \sqrt{\dotp{f}{f}}.
$$
Therefore $p(x)$ defines an Hilbert space $\mathcal{H} = (\{f: \R^d \rightarrow \R | \norm{f} < \infty \},\dotp{\cdot}{\cdot})$.
\subsection{Kernel as a Linear Operator}
A symmetric positive (semi)definite kernel $K$ defines a self-adjoint linear operator over the Hilbert space $\mathcal{H}$
$$
Kf(x) \eqdef \int_{\R^d} K(x,y)f(y)p(y) dy.
$$
The eigenfunctions of $K$ satisfies the following properties
\begin{align}
\label{kernel_eig}
(K\phi_i)(x) &\eqdef \int k(x,y)\phi_i(y)p(y) dy = \lambda_i \phi_i(x) ,\\
\label{ortho}
\dotp{\phi_i}{\phi_j} &\eqdef \int \phi_j(x)\phi_i(x)p(x) dx = \delta_{ij} .
\end{align}
Because the kernel is symmetric positive (semi)definite its eigenvalues are real and positive.
By convention we consider the eigenvalues as sorted in decreasing order $\lambda_1 \geq \lambda_2 \geq \dots \geq 0$.
\subsection{Low Rank Kernel Approximation}
The goal of low rank kernel approximation is to find a function $\psi: \R^d \rightarrow \R^m$ such that the kernel can be approximated with a finite dimensional dot product
$$
K(x,y) \approx \dotp{\psi(x)}{\psi(y)} .
$$
A natural way to quantify the approximation error is to take the expected value of the point-wise squared error:
\begin{align*}
E &\eqdef \E[(K(x,y) - \dotp{\psi(x)}{\psi(y)})^2] \\
&= \int_{\R^d}\int_{\R^d} (K(x,y) - \dotp{\psi(x)}{\psi(y)})^2 p(x,y) dx dy \\
&= \int_{\R^d}\int_{\R^d} (K(x,y) - \dotp{\psi(x)}{\psi(y)})^2 p(x) p(y) dx dy.
\label{error_formula}
\end{align*}
Bengio, Vincent and Paiement have proved in \cite{kern_eigfunctions} that
$$
\psi(x) = \left(\sqrt{\lambda_1}\phi_1(x),\sqrt{\lambda_2}\phi_2(x), \dots, \sqrt{\lambda_m}\phi_m(x)\right)
$$
minimizes the error $E$. This motivates the use of the leading eigenfunctions for approximating a kernel.


\section{Proposed Method}
The main idea behind IKA is to project the leading eigenfunctions of the kernel on an ``approximation space'' $\mathcal{F}$.
Then, by obtaining an explicit formulation of the Rayleigh quotient
over the space $\mathcal{F}$, it is possible to find the projections of the leading eigenfunctions by solving a generalized eigenvalue problem. 

\subsection{Derivation of the Method}
Let $ \mathcal{F} = Span\{b_1(x),b_2(x),\dots,b_n(x)\}$, where $b_i(x)$ are chosen to be linearly independent, be the ``approximation space''. Given a function $f \in \mathcal{F}$ we identify with $\vec{f}$ the only set of weights such that:
$$
f(x) = \summ{i}{n} \vec{f}_ib_i(x).
$$
We prove (in appendix A) that:
\begin{align}
\dotp{f}{g} &= \summ{i,j}{n} \vec{f}_i P_{ij} \vec{g}_j \eqdef \dotp{\vec{f}}{\vec{g}}_P, \\
\dotp{Kf}{f} &= \summ{i,j}{n} \vec{f}_i M_{ij} \vec{f}_j \eqdef \dotp{\vec{f}}{\vec{f}}_M
\end{align}
where
\begin{align*}
    P_{ij} &\eqdef \int_{\R^d}  b_i(x)b_j(x)p(x) dx, \\
    M_{ij} &\eqdef \int_{\R^d}\int_{\R^d} K(x,y) b_i(x)b_j(y) p(x)p(y) dydx.
\end{align*}
Given these results it is possible to write an explicit formulation of the Rayleigh quotient
$$
R(f) \eqdef\frac{\dotp{Kf}{f}}{\dotp{f}{f}} = \frac{\dotp{\vec{f}}{\vec{f}}_M}{\dotp{\vec{f}}{\vec{f}}_P},
$$
where $M,P \in \mathbb{M}_n(\R)$, $P$ is symmetric positive definite and $M$ is symmetric positive (semi)definite.
As a consequence the leading eigenfunctions can be approximated by solving the following generalized eigenproblem
\begin{align}
    M \vec{f} = \lambda P \vec{f}
    \label{eigenproblem}
\end{align}
which can be solved by many existing methods (see \cite{eigenproblem} and references therein).

\subsection{Numerical Approximation of P and M}
It is possible to approximate the matrices $P$ and $M$ with $\widetilde{P}$ and $\widetilde{M}$ by approximating the unknown density $p(x)$ with the empirical data density. Let $B \in \R^{N\times n}$ be a matrix with elements $B_{hi} = b_i(x_h)$ and assume it to have full rank. This assumption is reasonable because the functions $b_i$ are linearly independent therefore it is always possible to satisfy this assumption by providing enough sample points. 
\begin{align*}
    P_{ij} &\approx \frac{1}{N} \summ{h}{N} b_i(x_h)b_j(x_h) & & P \approx \frac{B^T B}{N} \eqdef \widetilde{P}, \\
    M_{ij} &\approx \frac{1}{N^2} \summ{h,k}{N} K(x_h,x_k) b_i(x_h)b_j(x_k) & &
    M \approx \frac{ B^T G B}{N^2} \eqdef \widetilde{M}.    
\end{align*}
Because the Gram matrix $G$ is symmetric positive (semi)definite and the matrix $B$ has full rank, $\widetilde{P}$ is symmetric positive definite and $\widetilde{M}$ is symmetric positive (semi)definite.
Therefore the leading eigenfunctions of the kernel $K$ can be approximated by solving the eigenproblem:
$$
\widetilde{M} \vec{f} = \lambda \widetilde{P} \vec{f}.
$$

\subsection{Implementation of the Proposed Method}
When dealing with large datasets computing the full Gram matrix $G \in \mathbb{M}_N(\R)$ can be unfeasible. Therefore to compute $\widetilde{P}$ and $\widetilde{M}$ we randomly sample $S$ points from the dataset.

\begin{algorithm}[H]
\SetAlgoLined
\SetKwInOut{Input}{input}
\SetKwInOut{Output}{output}
\Input{A positive (semi)definite kernel $K(x,y)$, a set of sampling points $\{y_i\}_{i=1}^S$ drawn randomly from the dataset and a set of basis functions $\{b_i(x)\}_{i=1}^n$}
\Output{A function $\psi(x)$ such that $K(x,y) \approx \dotp{\psi(x)}{\psi(y)}$}
\BlankLine
 compute the Gram matrix $G$ by setting $G_{i,j} = K(y_i, y_j)$ \;
 compute the matrix $B$ by setting $B_{ij} = b_j(y_i)$ \;
 $\widetilde{P} \gets \frac{B^T B}{S}$ \; 
 $\widetilde{M} \gets \frac{B^T G B}{S^2}$ \; 
 $(\lambda_1, \lambda_2, \dots, \lambda_n), (v^{(1)}, v^{(2)}, \dots, v^{(n)}) \gets \text{solve\_eigenproblem}(\widetilde{M},\widetilde{P})$ \;
\Return $\psi(x) = \left(\sqrt{\lambda_i} \summ{j}{n} v^{(i)}_j b_j(x) \right)_{i=1}^n$
 \caption{IKA}
\end{algorithm}

Because the sample size $S$ should be chosen to be $\gg n$, the numerical complexity of IKA is dominated by the operations on the matrix $G$. With a fixed number of filters $n$ IKA has numerical complexity of $O(S^2)$.

\section{Results}
\label{results}
We compare IKA against the Nystr\"om method \cite{nystrom} on the task of approximating the Gaussian kernel $K(x,y) = \exp\left(-\frac{\norm{x - y}^2}{2\sigma^2}\right)$ on a set of random patches sampled from the STL-10 dataset \cite{stl10}.
The parameter $\sigma^2$ is chosen to be the 10 percentile of $\norm{x -y}^2$ as in \cite{ckn}.
All the source code used to produce the results presented in this section and the full results of the experiments are available at \url{https://gitlab.com/matteo-ronchetti/IKA}. 

\subsection{Preprocessing}
\begin{enumerate}
    \item Each image $\widetilde{I}$ is normalized using Global Contrast Normalization:
    \begin{align*}
        I = \frac{\widetilde{I} - mean(\widetilde{I})}{\sqrt{var(\widetilde{I}) + 10}};
    \end{align*}
    \item 1'000'000 $7\times7$ patches are sampled at random locations from the images;
    \item PCA whitening is applied on the patches;
    \item Each patch is normalized to unit length.
\end{enumerate}
For IKA we separate the training and testing data with an 80/20 split.

\subsection{Effect of Sample Size}
We measure the effect of the sample size $S$ on the approximation error while using $n=128$ filters. \\
\begin{figure}[H]
\centering
\includegraphics[width=0.5\textwidth]{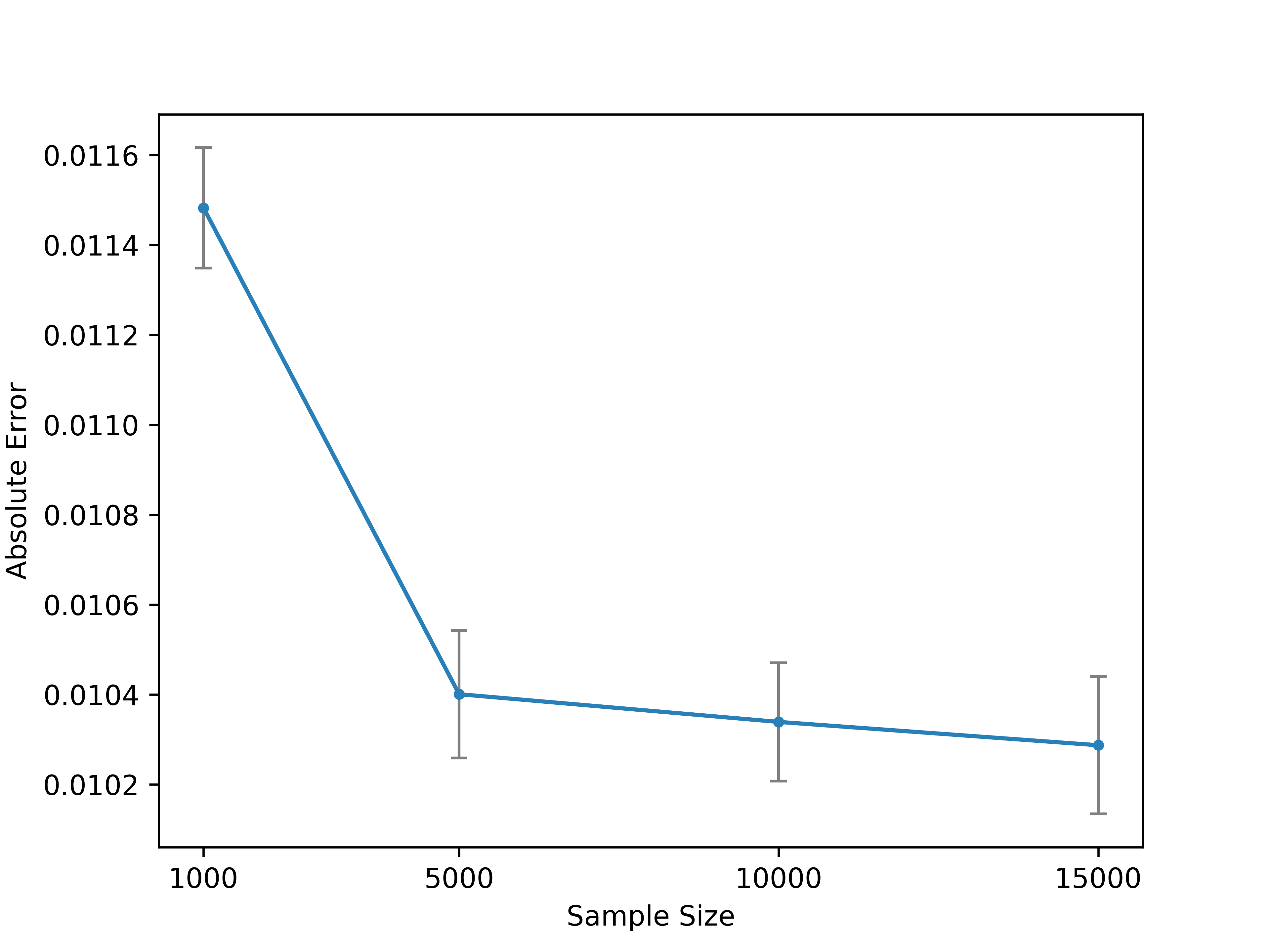} 
\end{figure}
As expected the use of a bigger sample size is beneficial. Notice that between 1'000 and 5'000 the error reduction is $\approx  9.4\%$ with approximately 25 times more operations. In contrast between 1'000 and 15'000 the error reduction is $\approx 10.4\%$ with approximately 225 times more operations. 

\subsection{Comparison with the Nystr\"om Method}
\subsubsection{Random Filters}
We compare IKA againts the Nystr\"om method. For IKA we use the best performing sample size ($S = 15000$), filters are chosen randomly between the sampled patches.
\begin{figure}[H]
\centering
\includegraphics[width=0.45\textwidth]{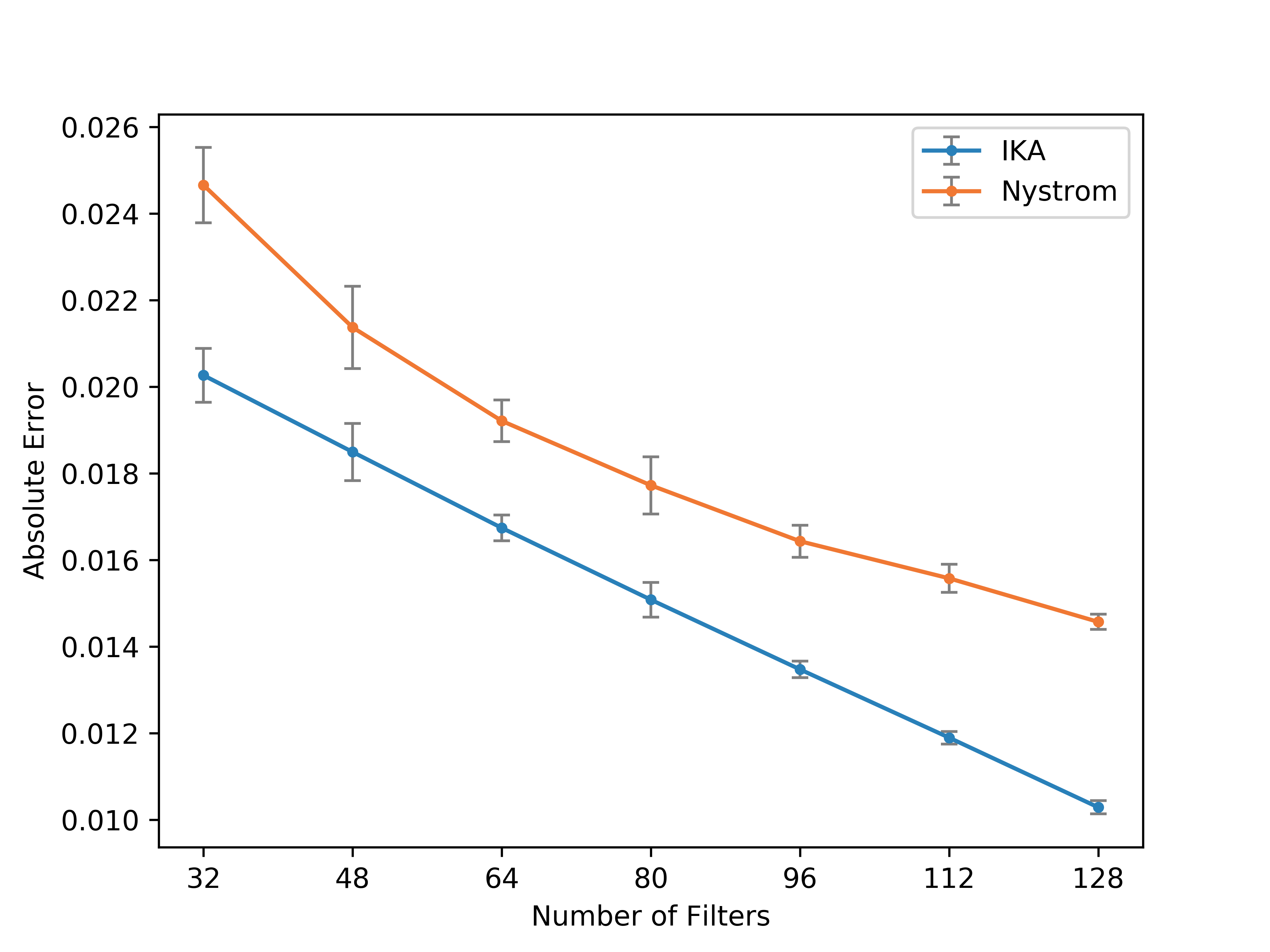} 
\includegraphics[width=0.45\textwidth]{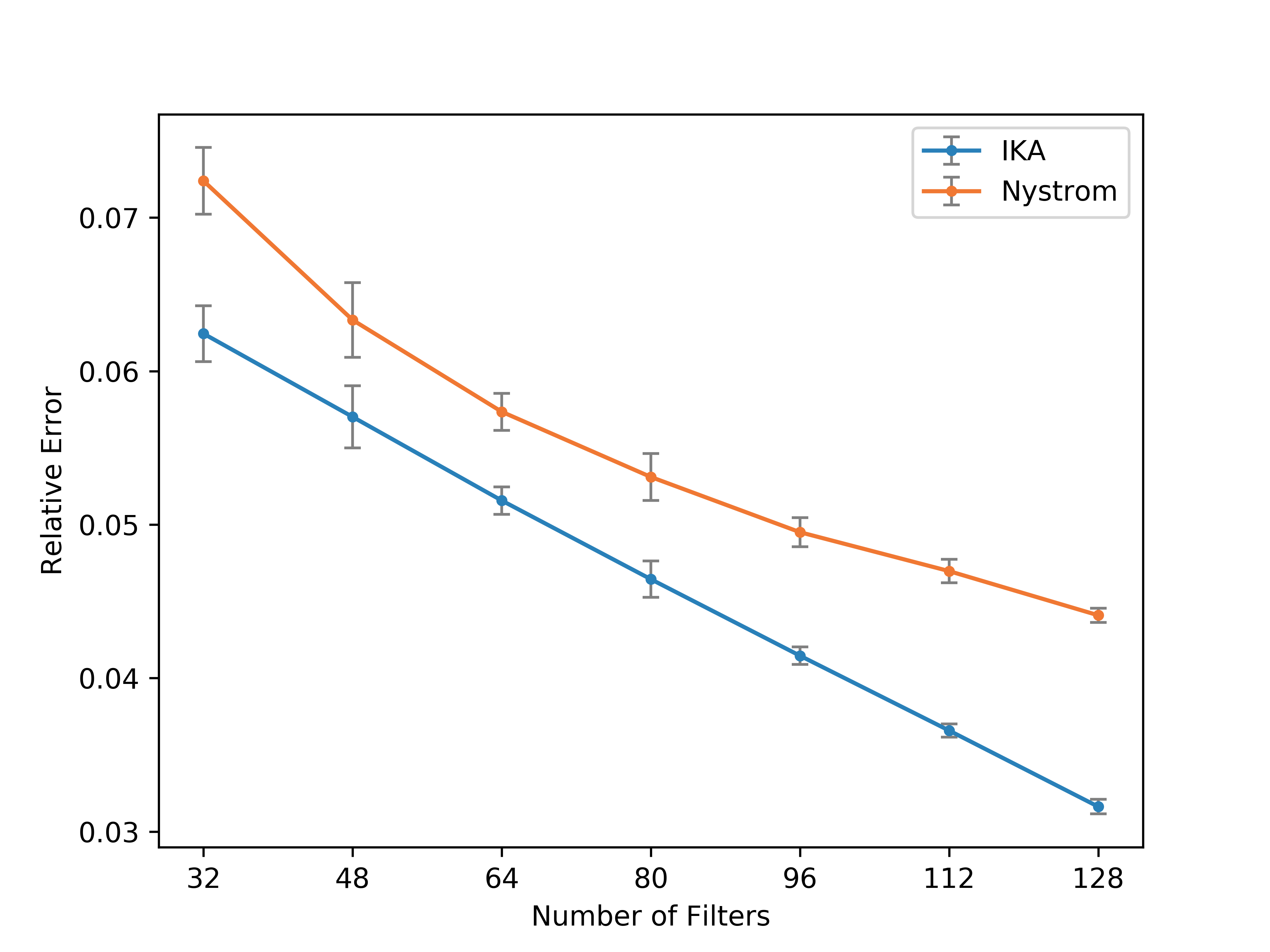} 
\end{figure}
The proposed method consistently outperforms the Nystr\"om method with mean reduction of absolute error of $\approx 18.6\%$.

\subsubsection{K-Means Filters}
It has been shown \cite{kmeans_nystrom} that choosing filters with k-means is beneficial for the accuracy of Nystr\"om method. Therefore we compare the methods using filters produced by mini-batch k-means and normalized to unit length.
\begin{figure}[H]
\centering
\includegraphics[width=0.45\textwidth]{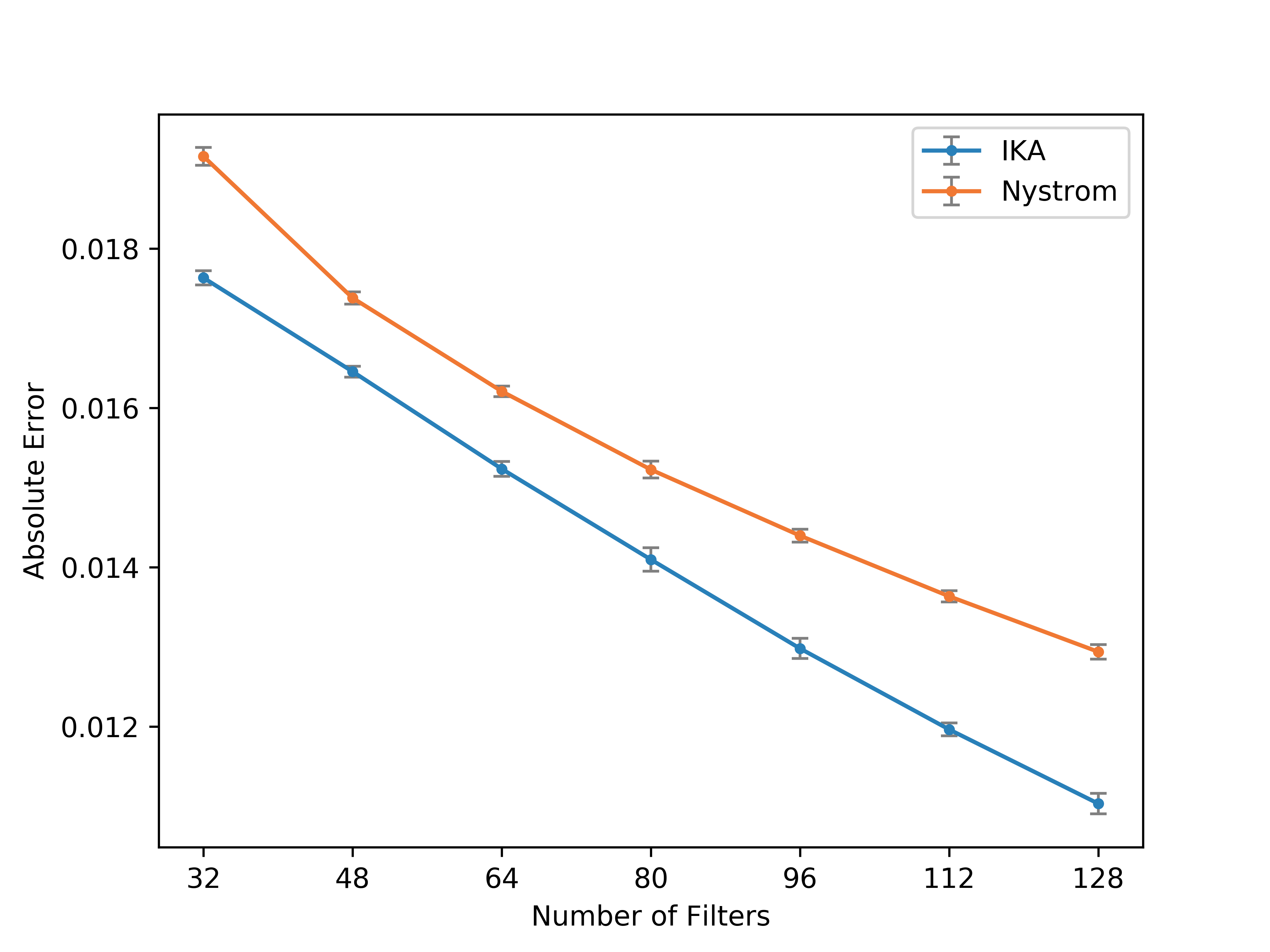} 
\includegraphics[width=0.45\textwidth]{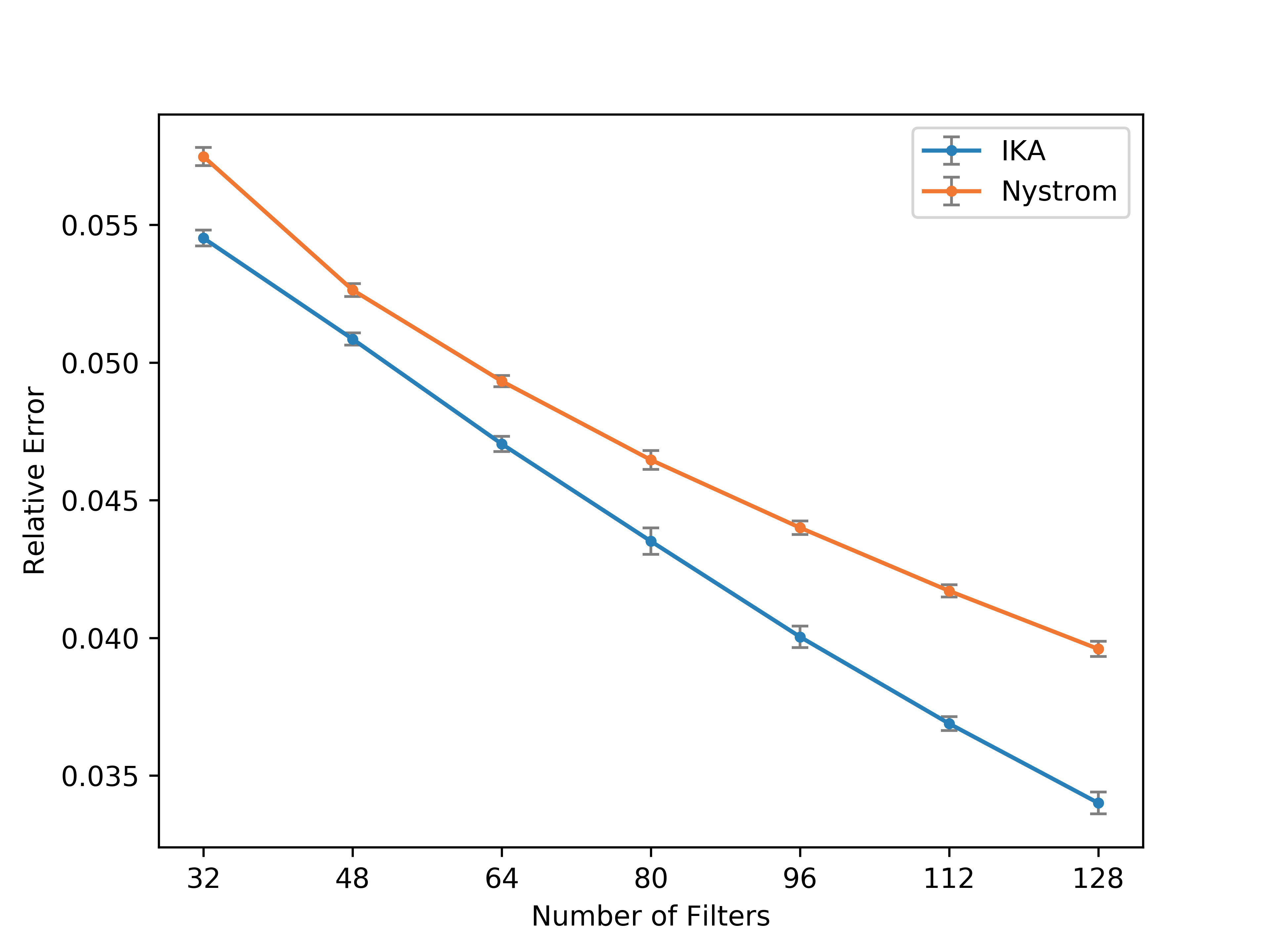} 
\end{figure}
The proposed method outperforms the Nystr\"om method with mean reduction of absolute error of $\approx 9.1\%$. The use of kmeans filters is more beneficial for Nystr\"om method than IKA.
Our intuition is that the the Nystr\"om method benefits from this choice of sampling points because they carry more information about the unknown density $p(x)$. In contrast IKA, which draws a bigger sample from  $p(x)$, is less affected by this benefit. 

\subsubsection{Effect of Number of Eigenfunctions}
In many practical applications it can be useful to fix the number of filters $n$ and only compute the first $m < n$ eigenfunctions.
Therefore we fix the number of filters to $n=128$ and observe the effect of $m$ on the approximation error.
\begin{figure}[h]
\centering
\includegraphics[width=0.5\textwidth]{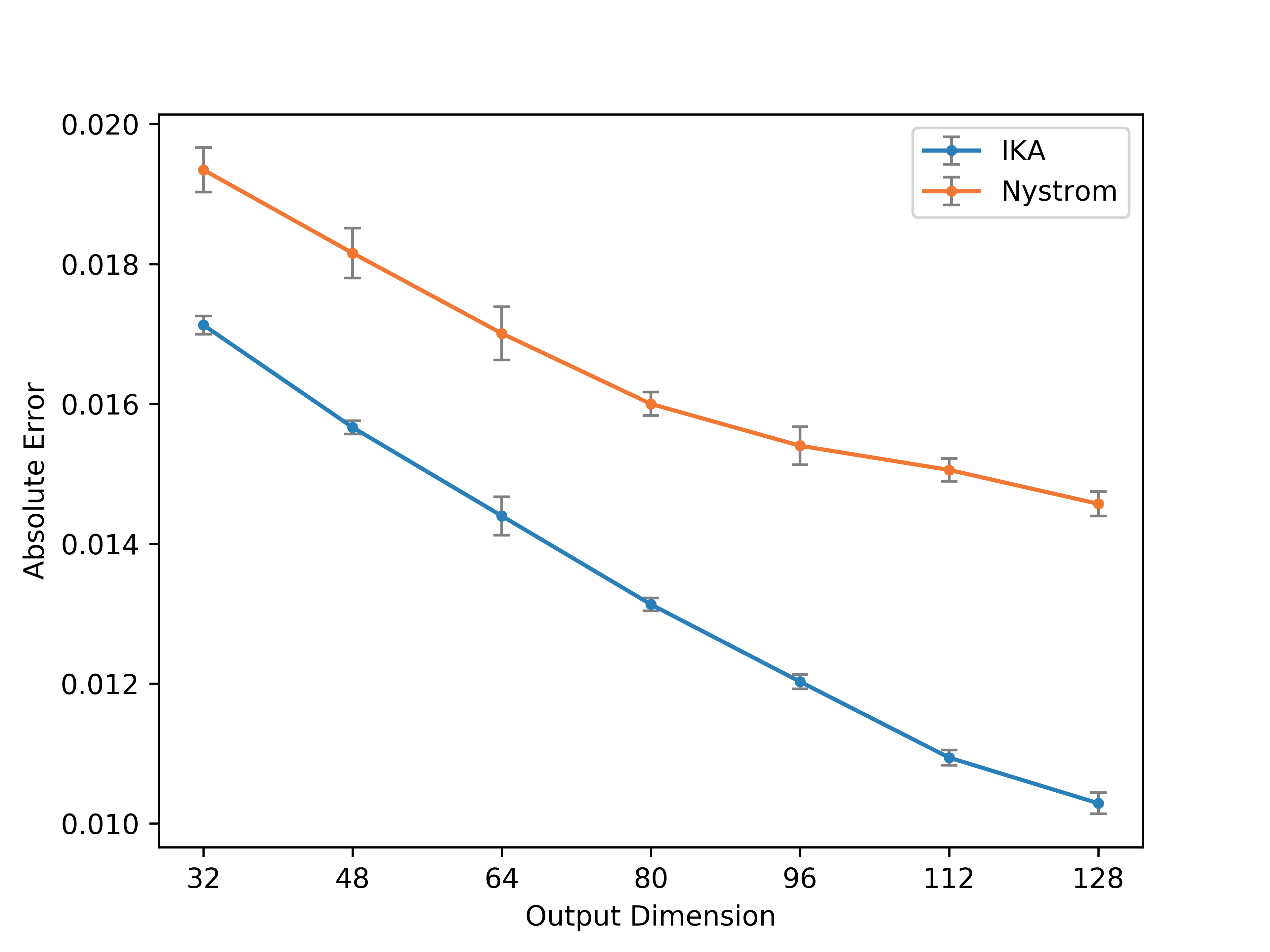} 
\end{figure}

\section{Related Work}
Related kernel approximation techniques include:
\begin{itemize}
    \item Techniques that exploit low-rank approximations of the kernel matrix such as Nystr\"om method \cite{nystrom, kmeans_nystrom};
    \item Random Fourier features
techniques for shift-invariant kernels \cite{random_features};
    \item Approximation through linear expansion of the Gaussian kernel \cite{ckn}.
\end{itemize}
See also \cite{serra} about eigenvalue distribution of kernel operators.

\newpage
\section{Conclusions}
We have proposed IKA, a new method for low rank kernel approximation. The key results described in this paper are:
\begin{itemize}
\item IKA produces a function $\psi(x) \in Span\{b_1(x),b_2(x),\dots,b_n(x)\}$ where the basis functions $b_i(x)$ can be arbitrarily chosen;
\item IKA outperformed the Nystr\"om method on a real world dataset, both when using random filters and filters chosen by kmeans.
\end{itemize}

The current work opens some future perspectives:
\begin{itemize}
    \item Study the use of different sets of basis functions. In particular using ReLu or Sigmoid activation functions;
    \item Design an algorithm that produces a good set of filters for IKA; 
    \item Study the performances of the proposed method on classification tasks (including multi-layer architectures).
\end{itemize}

\newpage
\printbibliography

\newpage
\appendix
\section*{Appendix A}
\begin{proposition}
The dot product (in $\mathcal{H}$) between two functions of $\mathcal{F}$ can be expressed as
$$\dotp{f}{g} = \summ{i,j}{n} \vec{f}_i P_{ij} \vec{g}_j$$
where the matrix $P \in \mathbb{M}_k(\R)$ is symmetric positive definite
\end{proposition}
\begin{proof}
\begin{align*}
\dotp{f}{g} &\eqdef \int_{\R^d} f(x)g(x)p(x)dx \\
&= \int_{\R^d}\left(\summ{i}{n}b_i(x)\vec{f}_i\right)\left(\summ{j}{n}b_j(x)\vec{g}_j\right)p(x)dx \\
&= \summ{i,j}{n} \vec{f}_i \vec{g}_j \underbrace{\int_{\R^d}  b_i(x)b_j(x)p(x) dx}_{= P_{i,j}} \\
&= \summ{i,j}{n} \vec{f}_i P_{i,j} \vec{g}_j \\
&\eqdef \dotp{\vec{f}}{\vec{g}}_P .
\end{align*}
We proceed to prove that the matrix $P \in \mathbb{M}_n(\R)$ with $P_{ij} = \dotp{b_i}{b_j}$ is symmetric positive definite.
From the definition is clear that $P$ is symmetric. Furthermore:
$$
\forall f \neq 0 \qquad
\dotp{\vec{f}}{\vec{f}}_P = \dotp{f}{f} \eqdef \int_{\R^d} f^2(x) p(x) dx > 0 .
$$
\end{proof}

\begin{proposition}
The dot product (in $\mathcal{H}$) between $f \in \mathcal{F}$ and $Kf$ can be expressed as
$$
\dotp{Kf}{f} = \summ{i,j}{n} \vec{f}_i M_{ij} \vec{f}_j
$$
\end{proposition}
\begin{proof}
The definition of $Kf$ can be expanded:
\begin{align*}
Kf(x) &\eqdef  \int_{\R^d} K(x,y)f(y)p(y)dy \\
&= \summ{i}{n} \vec{f}_i \int_{\R^d} K(x,y) b_i(y) p(y) dy .\end{align*}
Using this result together with the definition of dot product in $\mathcal{F}$ it is possible to obtain an explicit formulation for $\dotp{Kf}{f}$:
\begin{align*}
\dotp{Kf}{f} &\eqdef \int_{\R^d} Kf(x)f(x)p(x)dx \\
&= \summ{i}{n} \vec{f}_i \int_{\R^d} Kf(x)b_i(x) p(x)dx \\
&= \summ{i}{n}\summ{j}{n} \vec{f}_i \vec{f}_j \underbrace{\int_{\R^d}\int_{\R^d} K(x,y) b_i(x)b_j(y) p(x)p(y) dydx}_{=M_{ij}} \\
&= \dotp{\vec{f}}{\vec{f}}_M .
\end{align*}
\end{proof}

\end{document}